\documentclass[a4paper,twoside]{article}

\usepackage{epsfig}
\usepackage{subcaption}
\usepackage{calc}
\usepackage{amssymb}
\usepackage{amstext}
\usepackage{amsmath}
\usepackage{amsthm}
\usepackage{multicol}
\usepackage{pslatex}
\usepackage{apalike}
\usepackage{algorithm}
\usepackage{hyperref}
\usepackage{algorithmic}
\usepackage{booktabs}  
\usepackage{tabularx}  
\usepackage[bottom]{footmisc}
\usepackage{SCITEPRESS}     

\DeclareMathOperator*{\argmax}{arg\,max}
\DeclareMathOperator*{\argmin}{arg\,min}

\newtheorem{theorem}{Theorem}

\newtheorem{proposition}[theorem]{Proposition}

\begin{document}

\title{Optimizing Sensor Redundancy in Sequential Decision-Making Problems}

\author{\authorname{Jonas Nüßlein\sup{1}, Maximilian Zorn\sup{1}, Fabian Ritz\sup{1}, Jonas Stein\sup{1}, Gerhard Stenzel\sup{1}, Julian Schönberger\sup{1}, Thomas Gabor\sup{1} and Claudia Linnhoff-Popien\sup{1}}
\affiliation{\sup{1}Institute of Computer Science, LMU Munich, Germany}
\email{jonas.nuesslein@ifi.lmu.de}
}


\abstract{
Reinforcement Learning (RL) policies are designed to predict actions based on current observations to maximize cumulative future rewards. In real-world applications (i.e., non-simulated environments), sensors are essential for measuring the current state and providing the observations on which RL policies rely to make decisions.
A significant challenge in deploying RL policies in real-world scenarios is handling sensor dropouts, which can result from hardware malfunctions, physical damage, or environmental factors like dust on a camera lens. A common strategy to mitigate this issue is the use of backup sensors, though this comes with added costs. 
This paper explores the optimization of backup sensor configurations to maximize expected returns while keeping costs below a specified threshold, $C$. Our approach uses a second-order approximation of expected returns and includes penalties for exceeding cost constraints. We then optimize this quadratic program using Tabu Search, a meta-heuristic algorithm. The approach is evaluated across eight OpenAI Gym environments and a custom Unity-based robotic environment (\textit{RobotArmGrasping}). Empirical results demonstrate that our quadratic program effectively approximates real expected returns, facilitating the identification of optimal sensor configurations.
}

\onecolumn \maketitle \normalsize \setcounter{footnote}{0} \vfill

\section{\uppercase{Introduction}}

Reinforcement Learning (RL) has emerged as a prominent technique for solving sequential decision-making problems, paving the way for highly autonomous systems across diverse fields. From controlling complex physical systems like tokamak plasma \cite{degrave2022magnetic} to mastering strategic games such as Go \cite{silver2018general}, RL has demonstrated its potential to revolutionize various domains. However, when transitioning from controlled environments to real-world applications, RL faces substantial challenges, particularly in dealing with sensor dropouts. In critical scenarios, such as healthcare or autonomous driving, the failure of sensors can result in suboptimal decisions or even catastrophic outcomes.

The core of RL decision-making lies in its reliance on continuous, accurate observations of the environment, often captured through sensors. When these sensors fail to provide reliable data—whether due to hardware malfunctions, environmental interference, or physical damage—the performance of RL policies can degrade significantly \cite{dulac2019challenges}. This vulnerability highlights the importance of addressing sensor dropout to ensure the robustness of RL systems in real-world applications.

In domains such as aerospace, healthcare, nuclear energy, and autonomous vehicles, the consequences of sensor failures are particularly severe. A compromised sensor network in these fields can endanger lives, harm the environment, or result in costly failures. Thus, improving the resilience of RL systems to sensor dropouts is not just desirable—it is critical for ensuring their safe and reliable deployment in mission-critical applications.

One common approach to mitigate the risks of sensor dropouts is the implementation of redundant backup sensors. These backup systems provide an additional layer of security, stepping in when primary sensors fail, thereby maintaining the availability of crucial data. While redundancy can enhance system resilience, it also introduces significant costs, and not all sensor dropouts result in performance degradation severe enough to justify the investment in backups.

This paper presents a novel approach to optimizing backup sensor configurations in RL-based systems. Our method focuses on balancing the trade-off between system performance—quantified by expected returns in a Markov Decision Process (MDP)—and the costs associated with redundant sensors. Specifically, we aim to identify the most effective sensor configurations that maximize expected returns while ensuring costs remain within a specified threshold $C$. To the best of our knowledge, this is the first study to address the optimization of sensor redundancy in sequential decision-making environments. \footnote{This publication was created as part of the Q-Grid project (13N16179) under the “quantum technologies – from basic research to market” funding program, supported by the German Federal Ministry of Education and Research.}
\ \\
\ \\
The structure of this paper is organized as follows: First, we provide the \textbf{Background}, where we review key concepts, including Markov Decision Processes (MDPs) and the Quadratic Unconstrained Binary Optimization (QUBO) framework, which underpins our approach. In the \textbf{Algorithm} section, we formally define the problem of optimizing backup sensor configurations and introduce \textit{SensorOpt}, our proposed algorithm for solving this problem. Next, we discuss \textbf{Related Work}, situating our research within the broader context of RL-based decision-making, sensor reliability, and optimization strategies. The \textbf{Experiments} section begins with a proof-of-concept, demonstrating that our quadratic approximation accurately estimates the real expected return for different backup sensor configurations. We then present the results of our main experiments, which involve testing our approach across eight OpenAI Gym environments, followed by an evaluation in our custom-built Unity-based robotic arm environment, \textit{RobotArmGrasping}. These experiments validate the effectiveness and generalizability of our method.

\section{\uppercase{Background}}

\subsection{Markov Decision Processes (MDP)}

Sequential Decision-Making problems are frequently modeled as Markov Decision Processes (MDPs). An MDP is defined by the tuple $E = \langle S, A, T, r, p_0, \gamma \rangle$, where $S$ represents the set of states, $A$ is the set of actions, and $T(s_{t+1}\:|\:s_t,a_t)$ is the probability density function (pdf) that governs the transition to the next state $s_{t+1}$ after taking action $a_t$ in state $s_t$. The process is considered Markovian because the transition probability depends only on the current state $s_t$ and the action $a_t$, and not on any prior states $s_{\tau < t}$.

The function $r : S \times A \rightarrow \mathbb{R}$ assigns a scalar reward to each state-action pair $(s_t, a_t)$. The initial state is sampled from the start-state distribution $p_0$, and $\gamma \in [0,1)$ is the discount factor, which applies diminishing weight to future rewards, giving higher importance to immediate rewards.

A deterministic \textit{policy} $\pi : S \rightarrow A$ is a mapping that assigns an action to each state. The \textit{return} $R = \sum_{t=0}^{\infty} \gamma^{\;t} \cdot r(s_t, a_t)$ is the total (discounted) sum of rewards accumulated over an episode. The objective, typically addressed by Reinforcement Learning (RL), is to find an optimal policy $\pi^*$ that maximizes the expected cumulative return:

\begin{equation*}
\pi^* = \underset{\pi}{\arg\max} \:\: \mathbb{E}_{p_0} \: \left[\sum_{t=0}^{\infty} \gamma^{\;t} \cdot r(s_t, a_t) \: | \: \pi \right]
\end{equation*}

Actions $a_t$ are selected according to the policy $\pi$. In Deep Reinforcement Learning (DRL), where state and action spaces can be large and continuous, the policy $\pi$ is often represented by a neural network $\hat{f_{\phi}}(s)$ with parameters $\phi$, which are learned through training \cite{sutton2018reinforcement}.

\subsection{Quadratic Unconstrained Binary Optimization (QUBO)}

Quadratic Unconstrained Binary Optimization (QUBO) ~\cite{10} is a combinatorial optimization problem defined by a symmetric, real-valued $(m \times m)$ matrix $Q$, and a binary vector $x \in \mathbb{B}^m$. The objective of a QUBO problem is to minimize the following quadratic function:

\begin{equation}
x^* = \underset{x}{\arg\min} \: H(x) = \underset{x}{\arg\min} \: \sum_{i=1}^{m}\sum_{j=i}^{m}{x_i x_j Q_{ij}}
\end{equation}

The function $H(x)$ is commonly referred to as the \textit{Hamiltonian}, and in this paper, we refer to the matrix $Q$ as the ``QUBO matrix''.

The goal is to find the optimal binary vector $x^*$ that minimizes the Hamiltonian \cite{roch2023effect}. This task is known to be \textit{NP}-hard \cite{glover2018tutorial}, making it computationally intractable for large instances without specialized techniques. QUBO is a significant problem class in combinatorial optimization, as it can represent a wide range of problems. Moreover, several specialized algorithms and hardware platforms, such as quantum annealers and classical heuristics, have been designed to solve QUBO problems efficiently \cite{morita2008mathematical,farhi2014quantum,farhi2016quantum,nusslein2023solving,zielinski2023influence,nusslein2023black}.

Many well-known combinatorial optimization problems, such as Boolean satisfiability (SAT), the knapsack problem, graph coloring, the traveling salesman problem (TSP), and the maximum clique problem, have been successfully reformulated as QUBO problems~\cite{8,12,13,bucher2023dynamic,15,16}. This versatility makes QUBO a powerful tool for solving a wide variety of optimization tasks, including the one addressed in this paper.

\section{\uppercase{Algorithm}}

\subsection{Problem Definition}

Let $\pi$ be a trained agent operating within an MDP $E$. At each timestep, the agent receives an observation $o \in \mathbb{R}^u$, which is collected using $n$ sensors $\{s_i\}_{1 \leq i \leq n}$. Each sensor $s_i$ produces a vector $o^i$, and the full observation $o$ is formed by concatenating these vectors: $o = [o^1, o^2, \dots, o^n]$. The complete observation $o$ has dimension $|o| = \sum_i |o^i|$.

At the start of an episode, each sensor $s_i$ may drop out with probability $d_i \in [0,1]$, meaning it fails to provide any meaningful data for the entire episode. In the event of a dropout, the sensor’s output is set to $o^i = \textbf{0}$ for the rest of the episode.

If $\pi$ is represented as a neural network, it is evident that the performance of $\pi$, quantified as the expected return, will degrade as more sensors drop out. To mitigate this, we can add backup sensors. However, incorporating a backup sensor $s_i$ incurs a cost $c_i \in \mathbb{N}$, representing the expense associated with adding that redundancy.

The task is to find the optimal backup sensor configuration that maximizes the expected return while keeping the total cost of the backups within a predefined budget $C \in \mathbb{N}$. Let $x \in \mathbb{B}^n$ be a vector of binary decision variables, where $x_i = 1$ indicates the inclusion of a backup for sensor $s_i$. If $\mathbb{E}_{d,\pi,x}[R]$ represents the expected return while using backup configuration $x$, the optimization problem can be formulated with both a soft and hard constraint:

\begin{equation}
x^* = \underset{x}{\argmax} \: \mathbb{E}_{d,\pi,x}[R] \quad \text{s.t.} \quad \sum_{i=1}^n x_i c_i \leq C
\end{equation}
\ \\
An illustration of this problem is provided in \textit{Figure 1}. In the next section, we describe a method to transform this problem into a QUBO representation.

\subsection{SensorOpt}

There are potentially $2^n$ possible combinations of sensor dropout configurations, making it computationally expensive to evaluate all of them directly. Since we only have a limited number of episodes $B \in \mathbb{N}$ to sample from the environment $E$, we opted for a second-order approximation of $\mathbb{E}_{d,\pi,x}[R]$ to reduce the computational burden while still capturing meaningful interactions between sensor dropouts.

To compute this approximation, we first calculate the probability $q(d)$ that at most two sensors drop out in an episode:

\begin{equation*}
\begin{aligned}
    q(d) &= \prod_{i=1}^{n} (1 - d_i) + \sum_{i=1}^{n} d_i \cdot \prod_{j \neq i}^{n} (1 - d_j) \\
    &+ \sum_{i=1}^{n} \sum_{j < i}^{n} d_i d_j \cdot \prod_{l \neq i,j}^{n} (1 - d_l)
\end{aligned}
\end{equation*}

Next, we estimate the expected return $\hat{R}_{(i,j)}$ for each pair of sensors $(s_i, s_j)$ dropping out. This can be efficiently achieved using \textbf{Algorithm 1}, which samples the episode return $E^{(i,j)}(\pi)$ using policy $\pi$ with sensors $s_i$ and $s_j$ removed. The algorithm begins by calculating two return values for each sensor pair. It then iteratively selects the sensor pair with the highest \textit{momentum} of the mean return, defined as $|\overline{R_{(i,j)}[:-1]} - \overline{R_{(i,j)}}|$. Here, $\overline{R_{(i,j)}[:-1]}$ denotes the mean of all previously sampled returns for the sensor pair $(s_i, s_j)$, excluding the most recent return, while $\overline{R_{(i,j)}}$ represents the mean including the most recent return. The difference between these two values measures the momentum, capturing how much the expected return is changing as more episodes are sampled.

We base the selection of sensor pairs on momentum rather than return variance to differentiate between aleatoric and epistemic uncertainty \cite{valdenegro2022deeper}. Momentum is less sensitive to aleatoric uncertainty, making it a more reliable indicator in this context.

A simpler, baseline approach to estimate $\hat{R}_{(i,j)}$ would be to divide the budget $B$ of episodes evenly across all sensor pairs $(s_i, s_j)$. In this \textit{Round Robin} approach, we sample $k = \frac{B}{n(n+1)/2}$ episodes for each sensor pair and estimate $\hat{R}_{(i,j)}$ as the empirical mean. In the \textit{Experiments} section, we compare our momentum-based approach from \textit{Algorithm 1} against this naive Round Robin method to highlight its efficiency and accuracy.

Once we have estimated all the values of $\hat{R}_{(i,j)}$, we can compute the overall expected return $\hat{R}(d)$ for a given dropout probability vector $d$, again using a second-order approximation:

\begin{figure*}[t]
\centering
\minipage{1\textwidth}
  \centering
  \includegraphics[width=\linewidth]{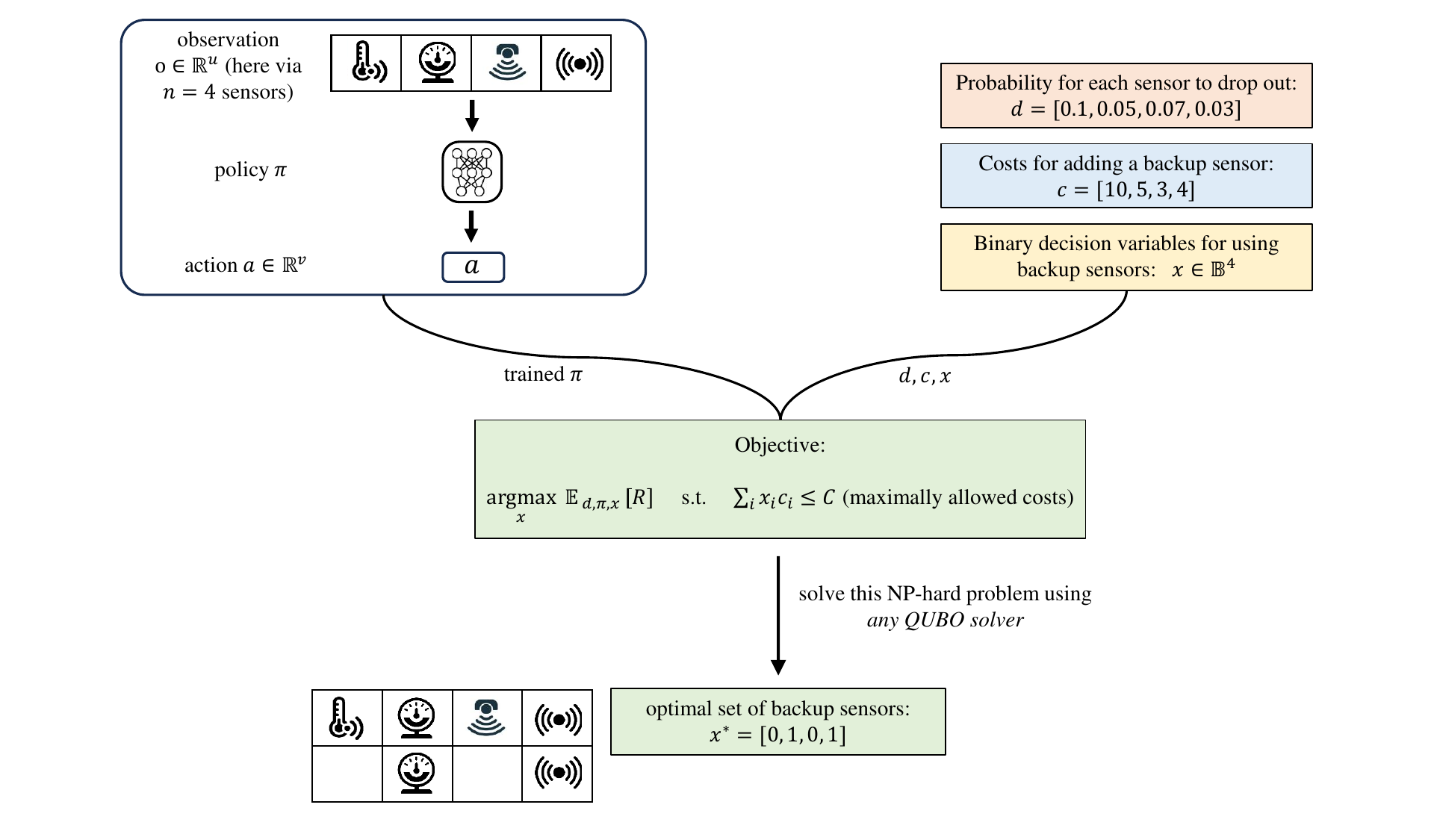}
  \caption{An illustration of our approach for optimizing the backup sensor configuration.}\label{fig:main_charts}
\endminipage
\vskip 0.2in
\end{figure*}

\begin{equation*}
\begin{split}
    \hat{R}(d) &= \frac{1}{q(d)} \left[ \hat{R}_0 \prod_{i=1}^{n} (1 - d_i) + \sum_{i=1}^{n} \hat{R}_{(i,i)} d_i \prod_{j \neq i}^{n} (1 - d_j) \right. \\
    &\quad + \left. \sum_{i=1}^{n} \sum_{j < i}^{n} \hat{R}_{(i,j)} d_i d_j \prod_{l \neq i,j}^{n} (1 - d_l) \right]
\end{split}
\end{equation*}
\ \\
Now we can calculate the advantage of using a backup for sensor $s_i$ vs. using no backup sensor for $s_i$:

\begin{equation*}
\Delta \hat{R}_{(i,i)}(d) = \hat{R}(d^{\{i\}}) -  \hat{R}(d)
\end{equation*}

with

\[
d^A =
\begin{cases}
d_i^2 & \text{if } i \in A, \\
d_i & \text{if } i \notin A.
\end{cases}
\]
\ \\
When a backup sensor is used for sensor $s_i$, the new dropout probability for $s_i$ becomes $d_i \gets d_i^2$, since the system will only fail to provide an observation if both the primary and backup sensors drop out. We adopt the notation $d^A$ to represent the updated dropout probabilities when backup sensors from set $A$ are in use. For example, if the dropout probability vector is $d = (0.2, 0.5, 0.1)$ and a backup is used for sensor 1, the updated vector becomes $d^{{1}} = (0.2, 0.25, 0.1)$, as the backup reduces the dropout probability for sensor 1 while sensors 0 and 2 remain unchanged.

\vskip 0.1in
\begin{algorithm}[h]
   \caption{Estimating $\hat{R}_{(i,j)}$}
   \label{alg:example}
\begin{algorithmic}
   \STATE {\bfseries Input:} budget of episodes to run $B \in \mathbb{N}$
   \STATE \:\:\:\:\:\:\:\:\:\:\:\:\: $\#$ sensors $n \in \mathbb{N}$
   \STATE \:\:\:\:\:\:\:\:\:\:\:\:\: MDP $E$
   \STATE \:\:\:\:\:\:\:\:\:\:\:\:\: trained policy $\pi$
   \STATE
   \STATE $R_{(i,j)} = [ E^{(i,j)}(\pi), E^{(i,j)}(\pi)$] \:\:\: $\forall \: i<j, (i,j) \in [1,n]^2 $
   \STATE 
   \FOR{$e=1$ {\bfseries to} $B - n(n+1)$}
   \STATE $(i,j) = \underset{(i,j)}{\argmax} \:\:\: | \: \overline{R_{(i,j)}[:-1]} - \overline{R_{(i,j)}} \: |$
   \STATE $R_{(i,j)}$.append$\big( E^{(i,j)}(\pi) \big)$
   \ENDFOR
   \STATE $\hat{R}_{(i,j)} = \overline{R_{(i,j)}}$ \hfill // {expected return = mean} 
   \STATE \hfill {empirical return}
   \STATE {\bfseries return} $\:$ all $\hat{R}_{(i,j)}$ \hfill
\end{algorithmic}
\end{algorithm}

We can now compute the joint advantage of using backups for both sensors $s_i$ and $s_j$. This joint advantage is not simply the sum of the individual advantages of using backups for each sensor in isolation:

\[
\Delta \hat{R}_{(i,j)}(d) = \hat{R}(d^{\{i,j\}}) -  \hat{R}(d) - \Delta \hat{R}_{(i,i)}(d) - \Delta \hat{R}_{(j,j)}(d)
\]

This formula accounts for the interaction between the two sensors, reflecting the fact that their joint contribution to the expected return is influenced by how the presence of both backup sensors affects the system as a whole, beyond just the sum of their individual contributions. This interaction term is crucial when optimizing backup sensor configurations, as it helps capture the diminishing or synergistic returns that can occur when multiple sensors are backed up together.

With these components in hand, we can now formulate the QUBO problem for the sensor optimization task defined in (2). The soft constraint $H^{soft}$ captures the optimization of the expected return, approximated by the change in return $\Delta \hat{R}_{(i,i)}(d)$ for individual sensors and $\Delta \hat{R}_{(i,j)}(d)$ for sensor pairs:

\[
H^{soft} = \sum_{i=1}^{n} x_i \cdot \Delta \hat{R}_{(i,i)}(d) + \sum_{i < j}^{n} x_i x_j \cdot \Delta \hat{R}_{(i,j)}(d)
\]

Here, $x_i$ are binary decision variables where $x_i = 1$ indicates that a backup sensor is added for sensor $s_i$. This formulation ensures that the algorithm optimizes the expected return based on the specific backup configuration.

To enforce the cost constraint, we introduce a hard constraint $H^{hard}$, which penalizes configurations where the total cost exceeds the specified budget $C$:

\[
H^{hard} = {\left(\sum_{i=1}^{n} x_i c_i + \sum_{i=n+1}^{n+1+\lceil \log_2(C) \rceil} x_i 2^{(i-n-1)} - C\right)}^2
\]

The expression $\sum_{i=1}^{n} x_i c_i$ represents the total cost of the selected backup sensors, and the second term accounts for the binary encoding of the cost constraint, ensuring that no configuration exceeds the budget $C$.

Finally, we introduce a scaling factor $\alpha$ to balance the contributions of the soft and hard constraints:

\[
\alpha = \sum_{i=1}^{n} | \Delta \hat{R}_{(i,i)}(d) | + \sum_{i < j}^{n} | \Delta \hat{R}_{(i,j)}(d) |
\]
\ \\
The overall QUBO objective function is then given by:
\begin{equation}
H = -H^{soft} + \beta \cdot \alpha \cdot H^{hard}
\end{equation}
\ \\
Here, $H^{soft}$ approximates the expected return $\mathbb{E}_{d,\pi,x}[R]$, while $H^{hard}$ ensures that the total cost remains within the allowable budget. The parameter $\beta$ is a hyperparameter that controls the trade-off between maximizing the expected return and enforcing the cost constraint.

Algorithm \textbf{SensorOpt} summarizes our approach for optimizing sensor redundancy configurations. This algorithm uses the QUBO formulation to find the optimal backup sensor configuration, balancing performance improvements with budget constraints.

\begin{algorithm}[h]
   \caption{\textbf{- SensorOpt}: Calculating best sensor backup configuration}
   \label{alg:example}
\begin{algorithmic}
   \STATE {\bfseries Input:} $\#$ sensors $n \in \mathbb{N}$
   \STATE \:\:\:\:\:\:\:\:\:\:\:\:\: dropout probability for each sensor $d \in [0,1]^n$
   \STATE \:\:\:\:\:\:\:\:\:\:\:\:\: costs for each backup sensor $c \in \mathbb{N}$
   \STATE \:\:\:\:\:\:\:\:\:\:\:\:\: maximally allowed costs $C \in \mathbb{N}$
   \STATE \:\:\:\:\:\:\:\:\:\:\:\:\: budget of episodes to run $B \in \mathbb{N}$
   \STATE \:\:\:\:\:\:\:\:\:\:\:\:\: MDP $E$
   \STATE \:\:\:\:\:\:\:\:\:\:\:\:\: trained policy $\pi$
   \STATE
   \STATE $\hat{R}_0 \gets E(\pi)$ \hfill // {expected return with no sensor dropouts}
   \STATE $\hat{R}_{(i,j)} \gets $ use \textbf{Algorithm 1} $(B, n, E, \pi)$
   \STATE $Q \gets$ create QUBO-matrix using \textbf{formula (3)}
   \STATE
   \STATE $x^* = \underset{x}{\argmin} \:\: x^T Q x$ \hfill // {solve using any QUBO solver,}
   \STATE \hfill {e.g. Tabu Search}
   \STATE \hfill  // {$x \in \mathbb{B}^m$, \:\: $m = n + \lceil log_2(C) \rceil$}
   \STATE 
   \STATE {\bfseries return} $\: x^*[:n]$ \hfill // {first $n$ bits of $x^*$ encode the optimal}
   \STATE \hfill {sensor backup configuration}
\end{algorithmic}
\end{algorithm}

For a given problem instance $(n, d, c, C, B, E, \pi)$, $\alpha$ is a constant while $\beta \in [0,1]$ is a hyperparameter we need to choose by hand. The number of boolean decision variables our approach needs is determined by $m = n + \lceil log_2(C) \rceil$.
\ \\

\begin{figure*}[t!]
\centering
\minipage{0.85\textwidth}
  \centering
  \includegraphics[width=\linewidth]{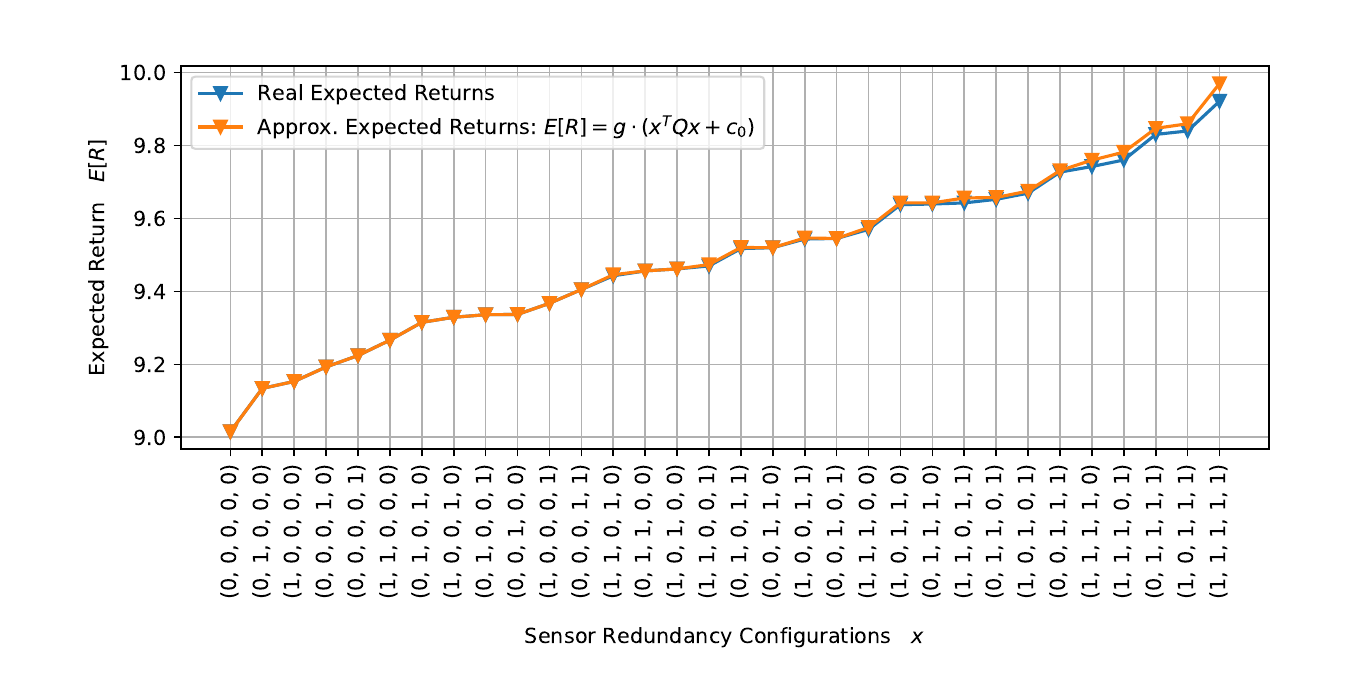}
  \caption{Proof-of-concept: this plot shows the real expected return and the approximated expected return $\mathbb{E}_{d,\pi,x}[R] \approx - \: x^T \: Q \: x + \hat{R}(d)$ when using a backup sensor configuration $x$. The configurations (x-axis) are sorted according to the real return.}\label{fig:main_charts}
\endminipage
\end{figure*}

\begin{proposition}
The optimization problem described in $(2)$ belongs to complexity class \textit{NP}-hard.
\end{proposition}
\begin{proof}
We can prove that the problem described in $(2)$ belongs to \textit{NP}-hard by reducing another problem $W$, that is already known to be \textit{NP}-hard, to it. Since this would imply that, if we have a polynomial algorithm for solving $(2)$, we could solve any problem instance $w \in W$ using this algorithm. We choose $W$ to be Knapsack \cite{salkin1975knapsack} which is known to be \textit{NP}-hard. For a given Knapsack instance $(n^{(1)}, v^{(1)}, c^{(1)}, C^{(1)})$ consisting of $n^{(1)}$ items with values $v^{(1)}$, costs $c^{(1)}$ and maximal costs $C^{(1)}$ we can reduce this to a problem instance $(n^{(2)}, d^{(2)}, c^{(2)}, C^{(2)}, E)$ from (2) via: $n^{(2)} = n^{(1)}$, $d^{(2)} = 0$, $c^{(2)} = c^{(1)}$, $C^{(2)} = C^{(1)}$. For each possible $x$ we define an individual \textit{MDP} $E_x$: $S = \{s_0, s_{terminal}\}, A = \{a_0\}, T(s_{terminal} | s_0, a_0) = 1, r(s_0,a_0) = x \cdot v$, $p_0(s_0) = 1$, $\gamma = 1$. The expected return $\mathbb{E}_{d,\pi,x}[R]$ will therefore be $\mathbb{E}_{d,\pi,x}[R] = x \cdot v$.
\end{proof}

\section{\uppercase{Related Work}}

Robust Reinforcement Learning \cite{moos2022robust,pinto2017robust} is a sub-discipline within RL that tries to learn robust policies in the face of several types of perturbations in the Markov Decision Process which are (1) \textit{Transition Robustness} due to a probabilistic transition model $T(s_{t+1} | s_t, a_t)$ (2) \textit{Action Robustness} due to errors when executing the actions (3) \textit{Observation Robustness} due to faulty sensor data.

There is a bulk of papers dealing with \textit{Observation Robustness} \cite{lutjens2020certified,mandlekar2017adversarially,pattanaik2017robust,zhang2020robust}. However, the vast majority is about handling noisy observations, meaning that the true state lies within an $\epsilon$-ball of the perceived observation. In these scenarios, an adversarial \textit{antagonist} is usually used in the training process that adds noise to the observation, trying to decrease the expected return, while the agent tries to maximize it \cite{liang2022efficient,lutjens2020certified,mandlekar2017adversarially,pattanaik2017robust,zhang2020robust}. In our problem setting, however, the sensors are not noisy in the sense of $\epsilon$-disturbance, but they can drop out entirely.

Another line of research is about the framework \textit{Action-Contingent Noiselessly Observable Markov Decision Processes} (ACNO-MDP) \cite{nam2021reinforcement,koseoglu2020miss}. In this framework, observing (measuring) the current state $s_t$ comes with costs. The agent therefore needs to learn when observing a state is crucially important for informed action decision making and when not. The action space is extended by two additional actions \textit{\{observe, not observe\}} \cite{krale2023act,beeler2021dynamic}. \textit{Not observe}, however, affects all sensors.

\begin{figure*}[t!]
\centering
\minipage{0.77\textwidth}
  \centering
  \includegraphics[width=\linewidth]{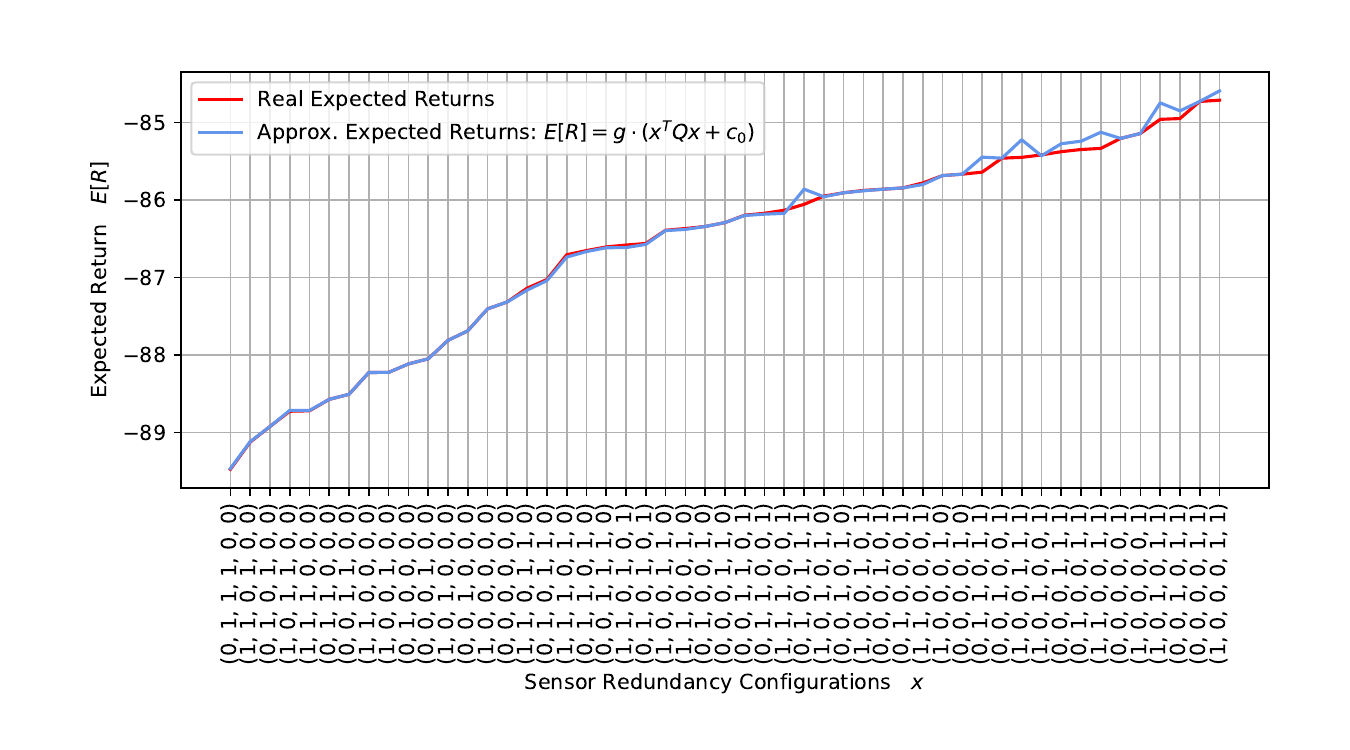}
  \includegraphics[width=\linewidth]{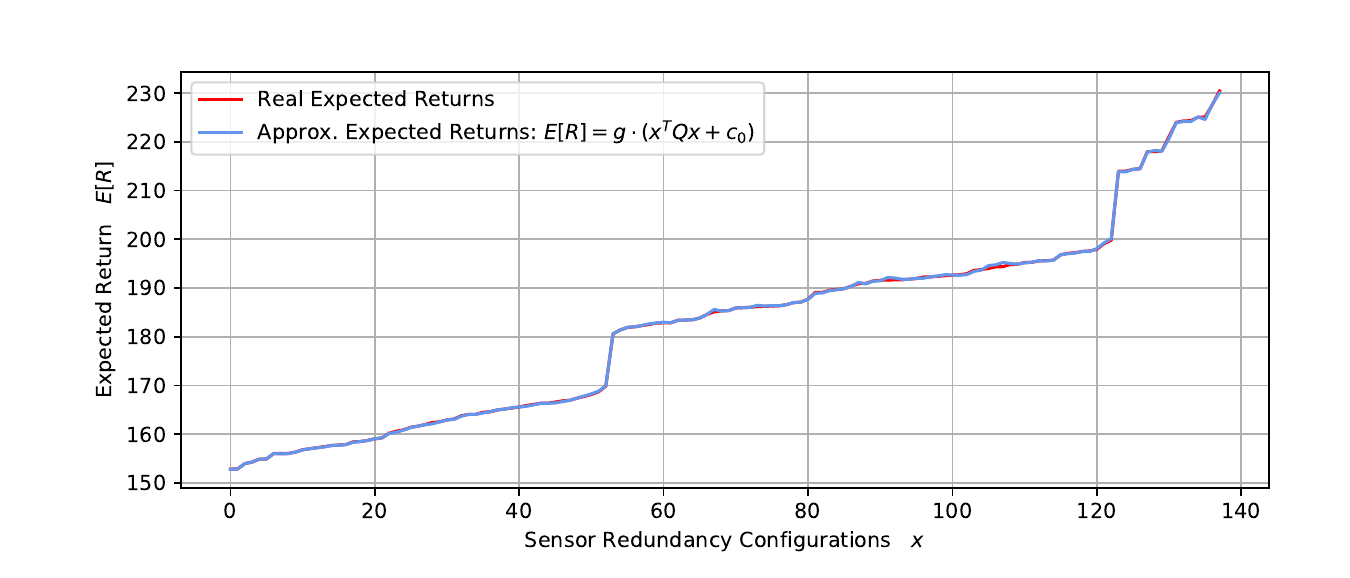}
  \includegraphics[width=\linewidth]{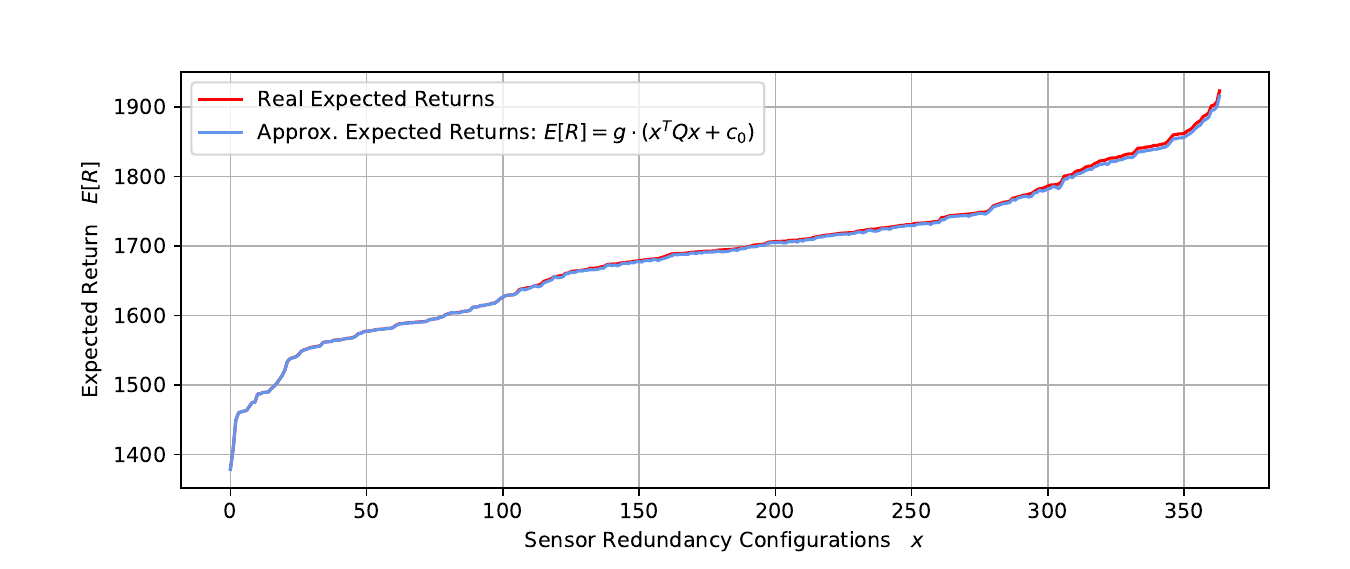}
  \caption{Solution landscapes of all possible backup sensor configurations for \textit{Acrobot-v1} (upper), \textit{LunarLander-v2} (middle) and \textit{Hopper-v2} (lower)}\label{fig:main_charts}
\endminipage
\end{figure*}

Important other related work comes from \textit{Safe Reinforcement Learning} \cite{gu2022review,dulac2019challenges}. In \cite{dulac2019challenges} the problem of sensors dropping out is already mentioned and analyzed. However, they only examine the scenario in which the sensor drops out for $z \in \mathbb{N}$ timesteps. Therefore they can use Recurrent Neural Networks for representing the policy that mitigates this problem.

Our approach for optimizing the sensor redundancy configuration given a maximal cost $C$ bears similarity to Knapsack \cite{salkin1975knapsack,martello1990knapsack}. In \cite{16} a QUBO formulation for Knapsack was already introduced and \cite{quintero2021characterizing} provides a study regarding the trade-off parameter for the hard- and soft-constraint. The major difference between the original Knapsack problem and our problem is that in Knapsack, the value of a collection of items is the sum of all item values: $W = \sum_i w_i x_i$. In our problem formulation this, however, doesn't hold. We therefore proposed a second-order approximation for representing the soft-constraint.

To the best of our knowledge, this paper is the first to address the optimization of sensor redundancy in sequential decision-making environments.

\section{\uppercase{Experiments}}

In this \textit{Experiments} section we want to evaluate our algorithm \textit{SensorOpt}. Specifically, we want to examine the following hypotheses:
\begin{enumerate}
    \item Hamiltonian $H$ of formula $(3)$ approximates the real expected returns.
    \item Our algorithm \textbf{SensorOpt} can find optimal sensor backup configurations.
    \item \textbf{Algorithm 1} approximates the expected returns $R_{(i,j)}$ faster than with \textit{Round Robin}.
\end{enumerate}

\subsection{Proof-of-Concept}

To test how well our QUBO formulation (a second-order approximation) approximates the real expected return for each possible sensor redundancy configuration $x$ we created a random problem instance, see Table 1, and determined the \textit{Real Expected Returns} $\mathbb{E}_{d,\pi,x}[R]$ and the \textit{Approximated Expected Returns}:

\[ \mathbb{E}_{d,\pi,x}[R] \approx - \: x^T \: Q \: x + \hat{R}(d)\]
\ \\
Note that in general a QUBO-matrix can be linearly scaled by a scalar $g$ without altering the order of the solutions regarding solution quality. \textit{Figure 2} shows the two resulting graphs where the x-axis represents all possible sensor backup configurations $x$ and the y-axis the expected return. We have sorted the configurations according to their real expected returns. The key finding in this plot is that the approximation is quite good even though it is not perfect for all $x$. But the relative performance is still intact, meaning especially that the best solution in our approximation is also the best solution regarding the real expected returns.

\vskip 0.15in
\begin{table}[h]
\caption{Parameters defining the proof-of-concept problem. Figure 2 shows the approximated and the real expected returns for each backup sensor configuration $x$.}
\begin{center}
\begin{tabular}{l|p{70mm}}  
\toprule
$n$    &   $5$ \\
$c$ &    $[4, 5, 3, 4, 2]$ \\
$C$    &   $390$ \\
$d$    &  $[0.09, 0.08, 0.1, 0.085, 0.095]$ \\
$R_0$     &    $10$ \\
$R_{ij}$      &  $\{(0, 0): 9, (0, 1): 4, (0, 2): 1, (0, 3): 4, (0, 4): 5, (1, 1): 9, (1, 2): 3, (1, 3): 5, (1, 4): 3, (2, 2): 7, (2, 3): 3, (2, 4): 2, (3, 3): 8, (3, 4): 4, (4, 4): 8\}$\\
\bottomrule
\end{tabular}
\end{center}
\end{table}
\ \\
We can therefore verify the first hypothesis that our Hamiltonian $H$ approximates the real expected returns. But as with any approximation, it can contain errors.

\begin{table*}[h!]
\minipage{1\textwidth}
\caption{Results of our algorithm \textit{SensorOpt} on $9$ different environments compared against the true optimum, determined via brute force and the baselines when using no additional backup sensors and all possible backup sensors. We solved the QUBO created in \textit{SensorOpt} using Tabu Search. Note that \textit{All Backups} is not a valid solution to the problem presented in formula (2), since the cost of using all backup sensors exceeds the threshold $C$. Due to the exponential complexity of conducting a brute force search for the optimal configuration, it was feasible to apply this approach only to the first four environments.}
\label{sample-table}
\vskip 0.01in
\begin{center}
\begin{small}
\begin{sc}
\begin{tabular}{lccccr}
\toprule
Environment & SensorOpt & Optimum* & No Backups & All Backups \\
\midrule
CartPole-v1     &  \textbf{493.3}&  \textbf{493.3}&  424.0& 498.7 \\
Acrobot-v1    &  \textbf{-86.73}&  \textbf{-86.73}&  -91.38& -83.01 \\
LunarLander-v2   &  \textbf{241.5}&  \textbf{241.5}&  200.1& 249.5 \\
Hopper-v2    &  \textbf{2674}&  \textbf{2674}&  1979& 3221 \\
HalfCheetah-v2 &  \textbf{9643}&  -&  8119& 10711 \\
Walker2D-v2    &  \textbf{3388}&  -&  2402& 4097 \\
BipedalWalker-v3      &  \textbf{215.5}&  -&  160.2& 273.0\\
Swimmer-v2      &  \textbf{329.6}&  -&  272.6& 341.7 \\
RobotArmGrasping   & \textbf{29.41}& -& 22.27&30.41 \\
\bottomrule
\end{tabular}
\end{sc}
\end{small}
\end{center}
\endminipage
\end{table*}

\subsection{Main experiments}

Next, we tested if \textit{SensorOpt} can find optimal sensor backup configurations in well-known MDPs. We used $8$ different OpenAI Gym environments \cite{brockman2016openai}. Additionally, we created a more realistic and industry-relevant environment \textit{RobotArmGrasping} based on Unity \cite{juliani2018unity} (see \textit{Figure 4 and 5}). In this environment, the goal of the agent is to pick up a cube and lift it to a desired location. The observation space is $20$-dimensional and the continuous action space is $5$-dimensional. More details about this environment can be found in subsection 5.3.
\ \\
In all experiments, we sampled each dropout probability $d_i$, cost $c_i$, and maximum costs $C$ uniformly out of a fixed set. For simplicity, we restricted the costs $c_i$ in our experiments to be integers.

\ \\
For the $8$ OpenAI Gym environments we used as the trained policies SAC- (for continuous action spaces) or PPO- (for discrete action spaces) models from \textit{Stable Baselines 3} \cite{stable-baselines3}. For our \textit{RobotArmGrasping} environment we trained a PPO agent \cite{schulman2017proximal} for 5M steps. We then solved the sampled problem instances using \textit{SensorOpt}. We optimized the QUBO in \textit{SensorOpt} using Tabu Search, a classical meta-heuristic. Further, we calculated the expected returns when using no backup sensors (\textit{No Backups}) and all $n$ backup sensors (\textit{All Backups}). For the four smallest environments \textit{CartPole}, \textit{Acrobot}, \textit{LunarLander} and \textit{Hopper} we also determined the true optimal backup sensor configuration $x^*$ using brute force. This was however not possible for larger environments.

For each environment, we sampled 10 problem instances. Note that using all backup sensors is not a valid solution regarding the problem definition in formula $(2)$ since the costs would exceed the maximally allowed costs $\sum_i c_i > C$. All results are reported in \textit{Table 2}.

The results show that \textit{SensorOpt} found the optimal sensor redundancy configuration in the first four environments for which we were computationally able to determine the optimum using brute force. In \textit{Figure 3} we have plotted the approximated and the real expected returns for the environments \textit{Acrobot-v1}, \textit{LunarLander-v2} and \textit{Hopper-v2} in a similar way to the proof-of-concept.

\subsection{RobotArmGrasping environment}

The \textit{RobotArmGrasping} domain is a robotic simulation based on a digital twin of the \textit{Niryo One}\footnote{\url{https://niryo.com/niryo-one/}} which is available in the \textit{Unity Robotics Hub}\footnote{\url{https://github.com/Unity-Technologies/Unity-Robotics-Hub}}.
It uses Unity's built-in physics engine and models a robot arm that can grasp and lift objects.

\vskip 0.1in
\begin{figure}[h!]
  \centering
  \includegraphics[width=0.9\linewidth]{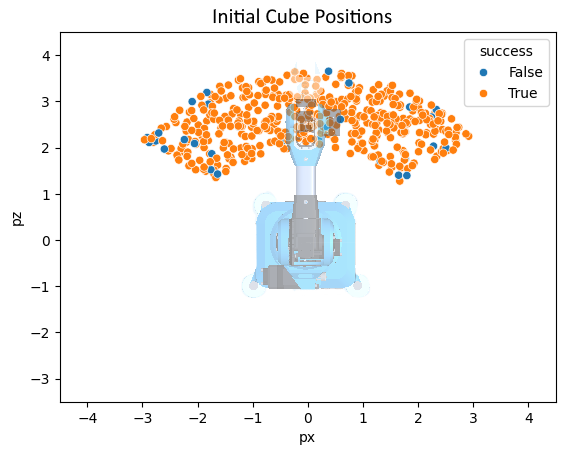}
  \caption{Illustration of the random cube positioning in the \textit{RobotArmGrasping} domain during training and evaluation. In this example, we traced 500 initial cube coordinates during the evaluation of an PPO agent. The color of the dots indicates whether the cube could successfully be grasped at these initial positions.}\label{fig:lift5d_initial_cube_positions}
\end{figure}
\vskip 0.1in

The robot arm is mounted to a table and consists of several controllable joints and a gripper.
The task consists of two steps.
First, the gripper has to be maneuvered close to the target object, a cube, which will cause the gripper to close automatically. 
If the gripper is positioned favorably, it will securely grasp the cube.
Second, the cube must be lifted upwards into a target area.
The robot arm's starting position is fixed.
The cube spawns on the table, randomly rotated, at a random position within a half circle around the robot arm (see Fig.~\ref{fig:lift5d_initial_cube_positions}).
The time for this task is limited to 500 steps.
The episode is reset if (a) the cube is lifted successfully into the target area, (b) the time limit is reached, (c) the robot hits either the table or its base, or (d) the cube is pushed out of the robot arm's range.

\begin{figure*}[t!]
\centering
\minipage{0.95\textwidth}
  \centering
  \includegraphics[width=\linewidth]{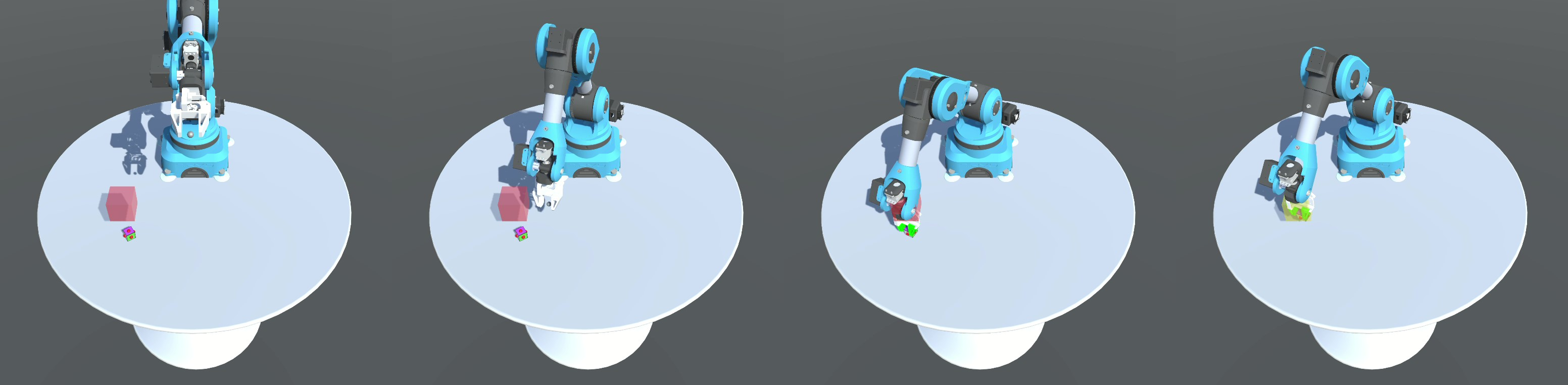}
  \caption{An illustration of our environment \textit{RobotArmGrasping}. }\label{fig:main_charts}
\endminipage
\end{figure*}

\begin{figure*}[t!]
\centering
\minipage{1\textwidth}
  \centering
  \includegraphics[width=\linewidth]{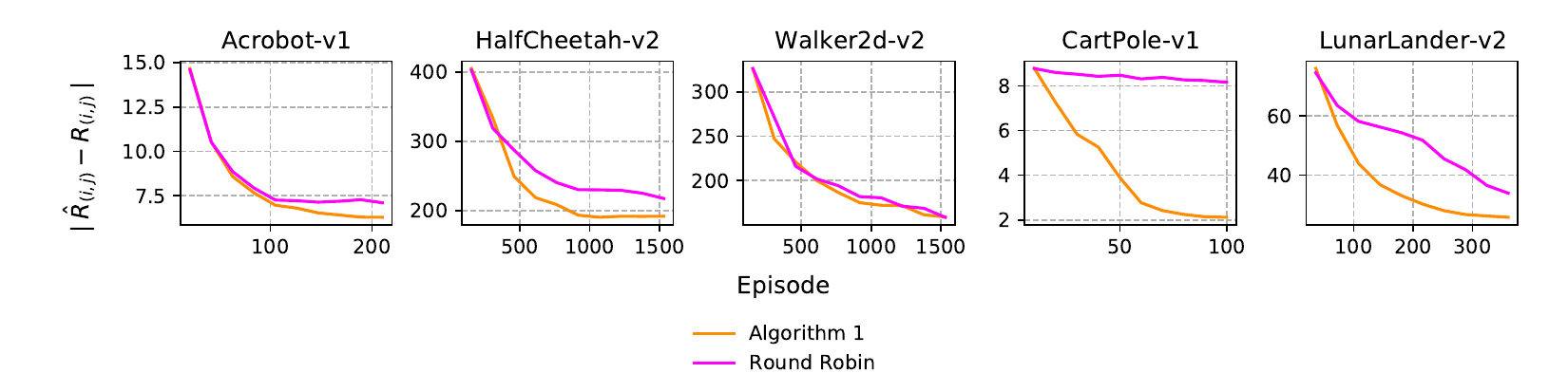}
  \caption{This figure shows a comparison between our \textit{Algorithm 1} and a \textit{Round Robin} approach for approximating the expected return $\hat{R}_{(i,j)}$ when dropping out sensors $s_i$ and $s_j$. The results indicate that it can be advantageous to prioritize pairs $(i,j)$ if their expected return $\hat{R}_{(i,j)}$ has a slow convergence rate and therefore dissecting aleatoric from epistemic uncertainty \cite{valdenegro2022deeper}.}\label{fig:main_charts}
\endminipage
\end{figure*}

To fulfill this task, the decision-making agent has access to the target position of 5 revolute joints: shoulder, arm, elbow, forearm, and wrist. We model this as a 5-dimensional, continuous vector $\in [-1;1]$. The agent may set a new target position per joint per step.
Each joint will then try to reach its target position with respect to its physical limits, e.g. its motor's maximum force, its dampening, its friction, and so on. Consequently, reaching the target may take (significantly) longer than one step.
We measure the decision-making agent's performance by a score function that evaluates the current simulation state. A random agent achieves a score of approx. $-20$ and has a success rate of $\ll 0.01$.
An agent trained with a reference implementation of PPO using SDE~\cite{stable-baselines3} for 5M steps achieves a score of $\thicksim15$ and a success rate of $\thicksim 0.94$.
Per step, we use the difference between the last and the current score as the reward.
Our trained PPO agent typically collects an (undiscounted) episode return of $\thicksim30$ if we allow collisions during evaluation.
Training such an agent takes approximately 4 hours (wall clock time) on a workstation equipped with an AMD Ryzen 7 5800X and 32 GB RAM. Figure 5 shows a visualization of a grasping attempt.

\subsection{Algorithm 1 vs. Round Robin}

\textit{Algorithm 1} is used to estimate the expected return $\hat{R}_{(i,j)}$ when dropping out sensors $s_i$ and $s_j$. As already mentioned in the \textit{Algorithm} section a naive way for doing that would be to use a \textit{Round Robin} approach \cite{rasmussen2008round} that divides the budget of episodes $B$ fairly among all $n(n+1)/2$ combinations. However, the larger the variance between the sampled returns for a fixed $(i,j)$, the more uncertain we are about the actual expected return. Deciding which $(i,j)$ to sample next from purely based on the variance of the returns is also not appropriate for estimating the mean due to aleatoric uncertainty. \textit{Algorithm 1} therefore decides which $(i,j)$ to sample next from based on the largest previous shift of the mean (momentum) since this is an indicator of how well the expected return $\hat{R}_{(i,j)}$ has converged. The third hypothesis we want to evaluate is ``Are we approximating $\hat{R}_{(i,j)}$ with \textit{Algorithm 1} faster than with \textit{Round Robin}?''. To test this hypothesis, we simply applied both approaches on $5$ environments and $10$ random problem instances each. As budget $B$ we chose $B = 10 \cdot n(n+1)/2$. So \textit{Round Robin} did evaluate each pair $(i,j)$ exactly $10$ times. As a metric, we measured the distance between the predicted expected return $\hat{R}_{(i,j)}$ and the actual expected return, determined by $100$ episodes. The result, see \textit{Figure 6}, suggests that the hypothesis is indeed true.

\section{\uppercase{Conclusion}}

Sensor dropouts present a major challenge when deploying reinforcement learning (RL) policies in real-world environments. A common solution to this problem is the use of backup sensors, though this approach introduces additional costs. In this paper, we tackled the problem of optimizing backup sensor configurations to maximize expected return while ensuring the total cost of added backup sensors remains below a specified threshold, $C$.

Our method involved using a second-order approximation of the expected return, $\mathbb{E}_{d,\pi,x}[R] \approx -x^T Q x + \hat{R}(d)$, for any given backup sensor configuration $x \in \mathbb{B}^n$. We incorporated a penalty for configurations that exceeded the maximum allowable cost, $C$, and optimized the resulting QUBO matrices $Q$ using the Tabu Search algorithm.

We evaluated our approach across eight OpenAI Gym environments, as well as a custom Unity-based robotic scenario, \textit{RobotArmGrasping}. The results demonstrated that our quadratic approximation was sufficiently accurate to ensure that the optimal configuration derived from the approximation closely matched the true optimal sensor configuration in practice.

\bibliographystyle{apalike}
{\small
\bibliography{example}}

\begin{thebibliography}{}

\bibitem[Beeler et~al., 2021]{beeler2021dynamic}
Beeler, C., Li, X., Bellinger, C., Crowley, M., Fraser, M., and Tamblyn, I. (2021).
\newblock Dynamic programming with incomplete information to overcome navigational uncertainty in a nautical environment.
\newblock {\em arXiv preprint arXiv:2112.14657}.

\bibitem[Brockman et~al., 2016]{brockman2016openai}
Brockman, G., Cheung, V., Pettersson, L., Schneider, J., Schulman, J., Tang, J., and Zaremba, W. (2016).
\newblock Openai gym.
\newblock {\em arXiv preprint arXiv:1606.01540}.

\bibitem[Bucher et~al., 2023]{bucher2023dynamic}
Bucher, D., N{\"u}{\ss}lein, J., O'Meara, C., Angelov, I., Wimmer, B., Ghosh, K., Cortiana, G., and Linnhoff-Popien, C. (2023).
\newblock Dynamic price incentivization for carbon emission reduction using quantum optimization.
\newblock {\em arXiv preprint arXiv:2309.05502}.

\bibitem[Choi, 2010]{13}
Choi, V. (2010).
\newblock Adiabatic quantum algorithms for the {NP}-complete maximum-weight independent set, exact cover and {3SAT} problems.

\bibitem[Choi, 2011]{8}
Choi, V. (2011).
\newblock Different adiabatic quantum optimization algorithms for the {NP}-complete exact cover and {3SAT} problems.

\bibitem[Degrave et~al., 2022]{degrave2022magnetic}
Degrave, J., Felici, F., Buchli, J., Neunert, M., Tracey, B., Carpanese, F., Ewalds, T., Hafner, R., Abdolmaleki, A., de~Las~Casas, D., et~al. (2022).
\newblock Magnetic control of tokamak plasmas through deep reinforcement learning.
\newblock {\em Nature}, 602(7897):414--419.

\bibitem[Dulac-Arnold et~al., 2019]{dulac2019challenges}
Dulac-Arnold, G., Mankowitz, D., and Hester, T. (2019).
\newblock Challenges of real-world reinforcement learning.
\newblock {\em arXiv preprint arXiv:1904.12901}.

\bibitem[Farhi et~al., 2014]{farhi2014quantum}
Farhi, E., Goldstone, J., and Gutmann, S. (2014).
\newblock A quantum approximate optimization algorithm.
\newblock {\em arXiv preprint arXiv:1411.4028}.

\bibitem[Farhi and Harrow, 2016]{farhi2016quantum}
Farhi, E. and Harrow, A.~W. (2016).
\newblock Quantum supremacy through the quantum approximate optimization algorithm.
\newblock {\em arXiv preprint arXiv:1602.07674}.

\bibitem[Glover et~al., 2018]{glover2018tutorial}
Glover, F., Kochenberger, G., and Du, Y. (2018).
\newblock A tutorial on formulating and using qubo models.
\newblock {\em arXiv preprint arXiv:1811.11538}.

\bibitem[Glover et~al., 2019]{15}
Glover, F., Kochenberger, G., and Du, Y. (2019).
\newblock Quantum bridge analytics {I}: A tutorial on formulating and using {QUBO} models.

\bibitem[Gu et~al., 2022]{gu2022review}
Gu, S., Yang, L., Du, Y., Chen, G., Walter, F., Wang, J., Yang, Y., and Knoll, A. (2022).
\newblock A review of safe reinforcement learning: Methods, theory and applications.
\newblock {\em arXiv preprint arXiv:2205.10330}.

\bibitem[Juliani et~al., 2018]{juliani2018unity}
Juliani, A., Berges, V.-P., Teng, E., Cohen, A., Harper, J., Elion, C., Goy, C., Gao, Y., Henry, H., Mattar, M., et~al. (2018).
\newblock Unity: A general platform for intelligent agents.
\newblock {\em arXiv preprint arXiv:1809.02627}.

\bibitem[Koseoglu and Ozcelikkale, 2020]{koseoglu2020miss}
Koseoglu, M. and Ozcelikkale, A. (2020).
\newblock How to miss data? reinforcement learning for environments with high observation cost.
\newblock In {\em ICML Workshop on the Art of Learning with Missing Values (Artemiss)}.

\bibitem[Krale et~al., 2023]{krale2023act}
Krale, M., Sim{\~a}o, T.~D., and Jansen, N. (2023).
\newblock Act-then-measure: Reinforcement learning for partially observable environments with active measuring.
\newblock {\em arXiv preprint arXiv:2303.08271}.

\bibitem[Liang et~al., 2022]{liang2022efficient}
Liang, Y., Sun, Y., Zheng, R., and Huang, F. (2022).
\newblock Efficient adversarial training without attacking: Worst-case-aware robust reinforcement learning.
\newblock {\em Advances in Neural Information Processing Systems}, 35:22547--22561.

\bibitem[Lodewijks, 2020]{12}
Lodewijks, B. (2020).
\newblock Mapping {NP}-hard and {NP}-complete optimisation problems to quadratic unconstrained binary optimisation problems.

\bibitem[Lucas, 2014]{16}
Lucas, A. (2014).
\newblock {Ising} formulations of many {NP} problems.

\bibitem[L{\"u}tjens et~al., 2020]{lutjens2020certified}
L{\"u}tjens, B., Everett, M., and How, J.~P. (2020).
\newblock Certified adversarial robustness for deep reinforcement learning.
\newblock In {\em conference on Robot Learning}, pages 1328--1337. PMLR.

\bibitem[Mandlekar et~al., 2017]{mandlekar2017adversarially}
Mandlekar, A., Zhu, Y., Garg, A., Fei-Fei, L., and Savarese, S. (2017).
\newblock Adversarially robust policy learning: Active construction of physically-plausible perturbations.
\newblock In {\em 2017 IEEE/RSJ International Conference on Intelligent Robots and Systems (IROS)}, pages 3932--3939. IEEE.

\bibitem[Martello and Toth, 1990]{martello1990knapsack}
Martello, S. and Toth, P. (1990).
\newblock {\em Knapsack problems: algorithms and computer implementations}.
\newblock John Wiley \& Sons, Inc.

\bibitem[Mooney et~al., 2019]{10}
Mooney, G., Tonetto, S., Hill, C., and Hollenberg, L. (2019).
\newblock Mapping {NP}-hard problems to restructed adiabatic quantum architectures.

\bibitem[Moos et~al., 2022]{moos2022robust}
Moos, J., Hansel, K., Abdulsamad, H., Stark, S., Clever, D., and Peters, J. (2022).
\newblock Robust reinforcement learning: A review of foundations and recent advances.
\newblock {\em Machine Learning and Knowledge Extraction}, 4(1):276--315.

\bibitem[Morita and Nishimori, 2008]{morita2008mathematical}
Morita, S. and Nishimori, H. (2008).
\newblock Mathematical foundation of quantum annealing.
\newblock {\em Journal of Mathematical Physics}, 49(12).

\bibitem[Nam et~al., 2021]{nam2021reinforcement}
Nam, H.~A., Fleming, S., and Brunskill, E. (2021).
\newblock Reinforcement learning with state observation costs in action-contingent noiselessly observable markov decision processes.
\newblock {\em Advances in Neural Information Processing Systems}, 34:15650--15666.

\bibitem[N{\"u}{\ss}lein et~al., 2023a]{nusslein2023black}
N{\"u}{\ss}lein, J., Roch, C., Gabor, T., Stein, J., Linnhoff-Popien, C., and Feld, S. (2023a).
\newblock Black box optimization using qubo and the cross entropy method.
\newblock In {\em International Conference on Computational Science}, pages 48--55. Springer.

\bibitem[N{\"u}{\ss}lein et~al., 2023b]{nusslein2023solving}
N{\"u}{\ss}lein, J., Zielinski, S., Gabor, T., Linnhoff-Popien, C., and Feld, S. (2023b).
\newblock Solving (max) 3-sat via quadratic unconstrained binary optimization.
\newblock In {\em International Conference on Computational Science}, pages 34--47. Springer.

\bibitem[Pattanaik et~al., 2017]{pattanaik2017robust}
Pattanaik, A., Tang, Z., Liu, S., Bommannan, G., and Chowdhary, G. (2017).
\newblock Robust deep reinforcement learning with adversarial attacks.
\newblock {\em arXiv preprint arXiv:1712.03632}.

\bibitem[Pinto et~al., 2017]{pinto2017robust}
Pinto, L., Davidson, J., Sukthankar, R., and Gupta, A. (2017).
\newblock Robust adversarial reinforcement learning.
\newblock In {\em International Conference on Machine Learning}, pages 2817--2826. PMLR.

\bibitem[Quintero and Zuluaga, 2021]{quintero2021characterizing}
Quintero, R.~A. and Zuluaga, L.~F. (2021).
\newblock Characterizing and benchmarking qubo reformulations of the knapsack problem.
\newblock Technical report, Technical Report. Department of Industrial and Systems Engineering, Lehigh~….

\bibitem[Raffin et~al., 2021]{stable-baselines3}
Raffin, A., Hill, A., Gleave, A., Kanervisto, A., Ernestus, M., and Dormann, N. (2021).
\newblock Stable-baselines3: Reliable reinforcement learning implementations.
\newblock {\em Journal of Machine Learning Research}, 22(268):1--8.

\bibitem[Rasmussen and Trick, 2008]{rasmussen2008round}
Rasmussen, R.~V. and Trick, M.~A. (2008).
\newblock Round robin scheduling--a survey.
\newblock {\em European Journal of Operational Research}, 188(3):617--636.

\bibitem[Roch et~al., 2023]{roch2023effect}
Roch, C., Ratke, D., N{\"u}{\ss}lein, J., Gabor, T., and Feld, S. (2023).
\newblock The effect of penalty factors of constrained hamiltonians on the eigenspectrum in quantum annealing.
\newblock {\em ACM Transactions on Quantum Computing}, 4(2):1--18.

\bibitem[Salkin and De~Kluyver, 1975]{salkin1975knapsack}
Salkin, H.~M. and De~Kluyver, C.~A. (1975).
\newblock The knapsack problem: a survey.
\newblock {\em Naval Research Logistics Quarterly}, 22(1):127--144.

\bibitem[Schulman et~al., 2017]{schulman2017proximal}
Schulman, J., Wolski, F., Dhariwal, P., Radford, A., and Klimov, O. (2017).
\newblock Proximal policy optimization algorithms.
\newblock {\em arXiv preprint arXiv:1707.06347}.

\bibitem[Silver et~al., 2018]{silver2018general}
Silver, D., Hubert, T., Schrittwieser, J., Antonoglou, I., Lai, M., Guez, A., Lanctot, M., Sifre, L., Kumaran, D., Graepel, T., et~al. (2018).
\newblock A general reinforcement learning algorithm that masters chess, shogi, and go through self-play.
\newblock {\em Science}, 362(6419):1140--1144.

\bibitem[Sutton and Barto, 2018]{sutton2018reinforcement}
Sutton, R.~S. and Barto, A.~G. (2018).
\newblock {\em Reinforcement learning: An introduction}.
\newblock MIT press.

\bibitem[Valdenegro-Toro and Mori, 2022]{valdenegro2022deeper}
Valdenegro-Toro, M. and Mori, D.~S. (2022).
\newblock A deeper look into aleatoric and epistemic uncertainty disentanglement.
\newblock In {\em 2022 IEEE/CVF Conference on Computer Vision and Pattern Recognition Workshops (CVPRW)}, pages 1508--1516. IEEE.

\bibitem[Zhang et~al., 2020]{zhang2020robust}
Zhang, H., Chen, H., Xiao, C., Li, B., Liu, M., Boning, D., and Hsieh, C.-J. (2020).
\newblock Robust deep reinforcement learning against adversarial perturbations on state observations.
\newblock {\em Advances in Neural Information Processing Systems}, 33:21024--21037.

\bibitem[Zielinski et~al., 2023]{zielinski2023influence}
Zielinski, S., N{\"u}{\ss}lein, J., Stein, J., Gabor, T., Linnhoff-Popien, C., and Feld, S. (2023).
\newblock Influence of different 3sat-to-qubo transformations on the solution quality of quantum annealing: A benchmark study.
\newblock In {\em Proceedings of the Companion Conference on Genetic and Evolutionary Computation}, pages 2263--2271.

\end{thebibliography}

\end{document}